%% file: main-arxiv.tex
\documentclass[12pt]{article}

\usepackage[export]{adjustbox}
\usepackage[margin=1in]{geometry}

\usepackage{xfrac}
\usepackage{times}
\usepackage{verbatim}
\usepackage{graphicx}
\usepackage{paralist}
\usepackage{amsmath,amsthm,amssymb,amsfonts}
\usepackage{mathrsfs}
\usepackage{listings}
\usepackage{booktabs}
\usepackage[colorlinks,linkcolor=black,bookmarksopen,
  bookmarksnumbered,citecolor=black,urlcolor=black]{hyperref}
\usepackage{authblk}
\usepackage{tikz}
\usepackage{algorithm}
\usepackage{algpseudocode}

\newcommand{\R}{\mathbb{R}}

\newcommand{\diam}{\texttt{diam}}

\newcommand{\Ucal}{\mathcal{U}}
\newcommand{\Vcal}{\mathcal{V}}

\newcommand{\Nrv}{\mathrm{Nrv}}

\newcommand{\Sf}{Steinhaus filtration}

\newcommand{\mymarginpar}[1]{}

\usepackage{enumitem,amssymb}
\newlist{todolist}{itemize}{2}
\setlist[todolist]{label=$\square$}
\usepackage{pifont}
\usepackage{color}

\usepackage[normalem]{ulem}

\definecolor{darkgrn}{rgb}{0, 0.8, 0}

\graphicspath{ {images/} }

\theoremstyle{plain}
\newtheorem{theorem}{Theorem}[section]
\newtheorem{lemma}[theorem]{Lemma}
\newtheorem{proposition}[theorem]{Proposition}

\theoremstyle{definition}
\newtheorem{definition}[theorem]{Definition}
\theoremstyle{remark}

\usepackage{subcaption}

\title{Steinhaus Filtration and Stable Paths in the Mapper}

\author{
 Dustin L. Arendt\footnote{Visual Analytics Group, Pacific Northwest National Laboratory, USA; \href{mailto:dustin.arendt@pnnl.gov}{dustin.arendt@pnnl.gov}}
 \quad
 Matthew Broussard\footnote{TD Bank, USA; \href{mailto:matthewbrouss@gmail.com}{matthewbrouss@gmail.com}}
 \quad
 Bala Krishnamoorthy\footnote{Department of Mathematics and Statistics, Washington State University, USA; \href{mailto:kbala@wsu.edu}{kbala@wsu.edu}}
 \newline
 Nathaniel Saul\footnote{Stripe, USA; \href{mailto:nat@riverasaul.com}{nat@riverasaul.com}}
 \quad
 Amber Thrall\footnote{Department of Mathematics and Statistics, Washington State University, USA; \href{mailto:amber.thrall@wsu.edu}{amber.thrall@wsu.edu}}
}

\begin{document}

\maketitle


\begin{abstract}
    \input{sections/abstract}
 \end{abstract}


\clearpage

\section{Introduction}
\input{sections/introduction}

\section{Steinhaus Filtrations}
\input{sections/definitions}

\section{Stability}
\input{sections/stability}

\section{Equivalence}
\input{sections/equivalence}

\section{Stable Paths}
\input{sections/paths}

\section{Applications}
\label{sec:applications}
\input{sections/applications}

\subsection{Recommendation Systems}
\label{ssec:recsys}
\input{sections/recommendation}

\subsection{Steinhaus Mapper Filtration}
\label{ssec:jcdmpr}

\input{sections/mapper}

\section{Discussion}
\input{sections/conclusion}

\section{Acknowledgment}
Broussard, Krishnamoorthy, and Saul acknowledge funding from the US National Science Foundation through grants 1661348 and 1819229.
Part of the research described in this paper was conducted under the Laboratory Directed Research and Development Program at Pacific Northwest National Laboratory, a multiprogram national laboratory operated by Battelle for the U.S.~Department of Energy.


\input{main-arxiv.bbltex}
\end{document}

%% file: sections/abstract.tex
We define a new filtration called the \emph{Steinhaus filtration} built from a single cover based on a generalized Steinhaus distance, a generalization of Jaccard distance.
The homology persistence module of a Steinhaus filtration with infinitely many cover elements may not be $q$-tame, even when the covers are in a totally bounded space.
While this may pose a challenge to derive stability results,
we show that the Steinhaus filtration is stable when the cover is finite.
We show that while the \v{C}ech and Steinhaus filtrations are not isomorphic in general, they are isomorphic for a finite point set in dimension one.
Furthermore, the VR filtration completely determines the $1$-skeleton of the Steinhaus filtration in arbitrary dimension.

We then develop a language and theory for stable paths within the Steinhaus filtration.
We demonstrate how the framework can be applied to several applications where a standard metric may not be defined but a cover is readily available.
We introduce a new perspective for modeling recommendation system datasets.
As an example, we look at a movies dataset and we find the stable paths identified in our framework represent a sequence of movies constituting a gentle transition and ordering from one genre to another. 

For explainable machine learning, we apply the Mapper algorithm for model induction by building a filtration from a single Mapper complex, and provide explanations in the form of stable paths between subpopulations.
For illustration, we build a Mapper complex from a supervised machine learning model trained on the FashionMNIST dataset.
Stable paths in the Steinhaus filtration provide improved explanations of relationships between subpopulations of images.

%% file: sections/introduction.tex
The need to rigorously seed a solution with a notion of stability in topological data analysis (TDA) has been addressed primarily using topological persistence \cite{Ca2009,ChMi2021,Gh2008barcodes}.
Persistence arises when we work with a sequence of objects built on a dataset, a \emph{filtration}, rather than with a single object.
One line of focus of this work has been on estimating the homology of the dataset. 
This typically manifests itself as examining the persistent homology represented as a diagram or barcode, with interpretations of zeroth and first homology as capturing significant clusters and holes, respectively \cite{AdCa2009,EdHa2009,EdLeZo2002,Zo2005}.
In practice it is not always clear how to interpret higher dimensional homology (even holes might not make obvious sense in certain cases).
A growing focus is to use persistence diagrams 
to help compare different datasets rather than interpret individual homology groups \cite{Adetal2017,ChCoGuMeOu2009,TuMiMuHa2014}.

The implicit assumption in most such TDA applications is that the data is endowed with a natural metric, e.g., points exist in a high-dimensional space or pairwise distances are available.
In certain applications, it is also not clear how one could assign a meaningful metric.
For example, memberships of people in groups of interest is captured simply as sets specifying who belongs in each group.
An instance of such data is that of recommendation systems, e.g., as used in Netflix to recommend movies to the customer.
Graph based recommendation systems have been an area of recent research.
Usually these systems are modeled as a bipartite graph with one set of nodes representing recommendees and the other representing recommendations.
In practice, these systems are augmented in bespoke ways to accommodate whichever type of data is available.
It is highly desirable to analyze the structure without prescribing some ill-fit or incomplete metric to the data. 

Another TDA approach for structure discovery and visualization of high-dimensional data is based on a construction called \emph{Mapper} \cite{SiMeCa2007}.
Defined as the nerve of a refined pullback cover of the data (See Figure \ref{fig:circlemapper}), Mapper has found increasing use in diverse applications in the past several years \cite{Lumetal2013}.
Attention has recently focused on interpreting parts of the $1$-skeleton of the Mapper complex, which is a simplicial complex, as significant features of the data.
Paths, flares, and cycles have been investigated in this context \cite{Lietal2015,NiLeCa2011,ToOlThRaCuSc2016}.
The framework of persistence has been applied to this construction to define a \emph{multi-scale} Mapper, which permits one to derive results on stability of such features \cite{DeMeWa2016}.
At the same time, the associated computational framework remains unwieldy and still most applications base their interpretations on a single Mapper object.

\begin{figure}[htp!]
 \includegraphics[width=\textwidth]{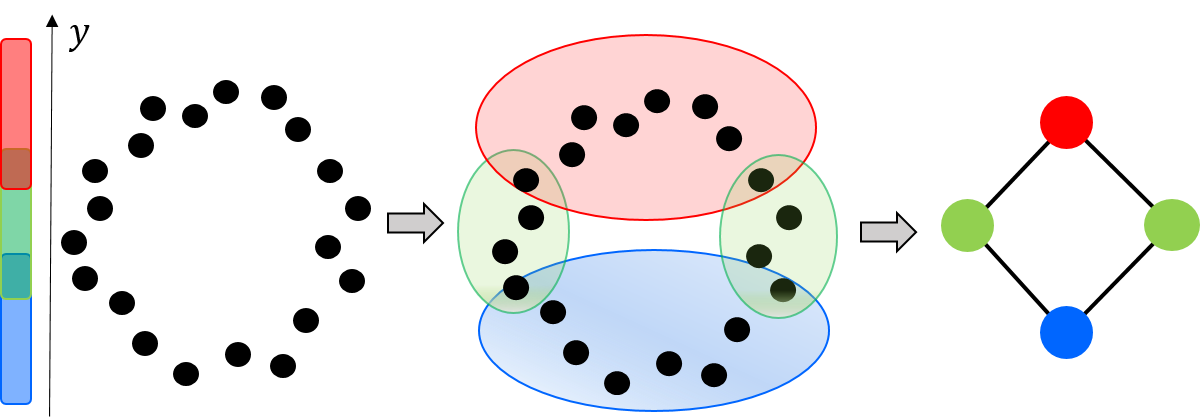}
 \caption{Mapper (Right) constructed on a noisy point set sampled from a circle (Left).
 The cover consists of three overlapping intervals (in blue, green, red) covering the range of $y$-coordinate values of the points.
 }\label{fig:circlemapper}
\end{figure}

Note that the Mapper construction works with covers.
The default approach is to start with overlapping hypercubes that cover a parameter space, which is usually a subset of $\R^d$ for some dimension $d$, and consider the pullback of this cover to the space of data.
In recommendation systems, the cover is just a collection of abstract sets providing membership info.
Could we define a topological construction on such abstract covers that still reveals the topology of the dataset?

We could study paths in this construction, but as the topological constructions are noisy, we would want to define a notion of stability for such paths.
With this goal in mind, could we define a \emph{filtration} from the abstract cover?
But unlike in the setting of, e.g., multiscale Mapper \cite{DeMeWa2016}, we do not have a sequence of covers (called a tower of covers)---we want to work with a \emph{single} cover.
How do we define a filtration on a single abstract cover?
Could we prove stability results for such a filtration?
Finally, could we demonstrate the usefulness of our construction on real data?

\subsection{Our Contributions} \label{ssec:contrib}

We introduce a new type of filtration defined on a single abstract cover.
Termed \emph{Steinhaus filtration}, our construction uses Steinhaus distances between elements of the cover.
We generalize the Steinhaus distance between two elements to those of multiple elements in the cover, and define a filtration on a single cover using the generalized Steinhaus distance as the filtration index.
The homology persistence module of a Steinhaus filtration with infinitely many cover elements may not be $q$-tame, even when the covers are in a totally bounded space (see Lemma \ref{lem:nonqtame}).
While this may pose a challenge to derive stability results,
we show a stability result on the Steinhaus filtration when the cover is finite---the largest change in generalized Steinhaus distance provides an upper bound for the bottleneck distance of their homologies (see Theorem \ref{thm:stability}).
We prove that in 1-dimension the Steinhaus filtration is isomorphic to the standard \v{C}ech filtration built on the dataset (see Theorem \ref{thm:cecheqv1d}) and independently that the Vietoris-Rips (VR) filtration completely determines the $1$-skeleton of the Steinhaus filtration in arbitrary dimensions (see Lemma \ref{lem:vrdetstnhs}).
We show by example that the Steinhaus filtration is not isomorphic to the \v{C}ech filtration in higher dimensions (see Proposition \ref{prop:cech_cover_not_iso}).

This filtration is quite general, and enables the computation of persistent homology for datasets without requiring strong assumptions or defining ill-fit metrics.
With real life applications in mind, we study paths in our construction.
Paths provide intuitive explanations of the relationships between the objects that the terminal vertices represent. 
Our perspective of path analysis is that shortest may not be more descriptive---see Figure \ref{fig:SPnotstable} for an illustration. 
Instead, we define a notion of \emph{stability} of paths in the Steinhaus filtration.
Under this notion, a stable path is analogous to a highly persistent feature as identified by persistent homology.

\begin{figure}[ht!]
  \centering
  \includegraphics[width=0.8\textwidth]{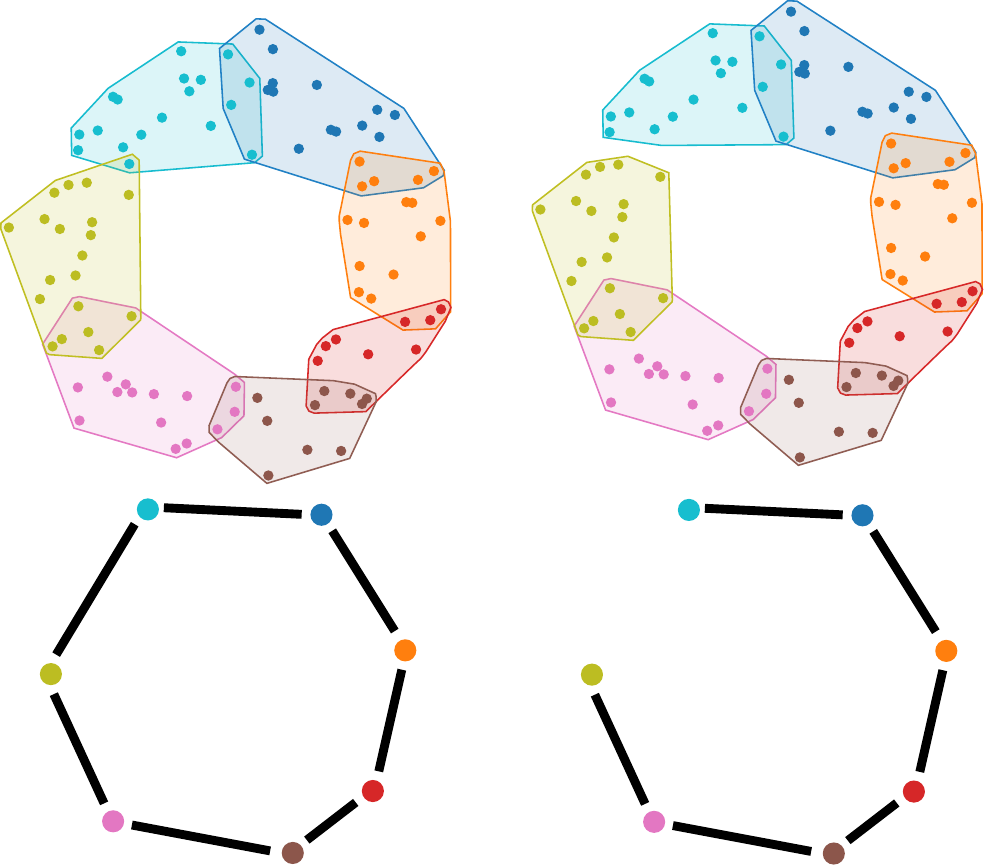}
  \caption{
        A cover with $7$ elements, and its nerve (left column).
        The cyan and green vertices are connected by a single edge
        generated by a \emph{single} point in intersection of cyan and green cover elements.
        Removing this point gives the cover and nerve shown in the right column.
        The path from cyan to green node now has six edges.
  }
  \label{fig:SPnotstable}
\end{figure}

We demonstrate the utility of stable paths in Steinhaus filtrations on two real life applications: a problem in movie recommendation system and Mapper.
We first show how recommendation systems can be modeled using the Steinhaus filtration, and then show how stable paths within this filtration suggest a sequence of movies that represent a ``smooth'' transition from one genre to another  (Section \ref{ssec:recsys}).
We then define an extension of the traditional Mapper \cite{SiMeCa2007} termed the \emph{Steinhaus Mapper Filtration}, and show how its stable paths provide valuable explanations of populations in the Mapper complex, focusing on explainable machine learning (Section \ref{ssec:jcdmpr}).
Code for our computations is available at \href{https://github.com/AmberThrall/SteinhausNotebooks/}{https://github.com/AmberThrall/SteinhausNotebooks/}.

\subsection{Related Work} \label{ssec:relwork}

Cavanna and Sheehy \cite{cavannapersistent} developed theory for a \emph{cover filtration}, built from a cover of a filtered simplicial complex.
But we work from more general covers of arbitrary spaces. 

We are inspired by similar goals as those of previous work addressing stability of the Mapper construction by Dey, M\'emoli, and Wang \cite{DeMeWa2016} and by Carri\`ere and Oudot \cite{CaOu2018}.
Our goal is to provide some consistency, and thus interpretability, to the Mapper construction. 
We incorporate ideas of persistence in a different manner into our construction using a \emph{single} cover, which considerably reduces the effort of generating results. 

The multi-scale Mapper defined by Dey, M\'emoli, and Wang \cite{DeMeWa2016} builds a filtration on the Mapper complex by varying the parameters of a cover. 
This construction yields nice stability properties, but is unwieldy in practice and difficult to interpret. 
Carri\`ere, Michel, and Oudot develop ways based on extended persistence to automatically select a hypercube cover that best captures the topology of the data \cite{CaMiOu2018}. 
This approach constructs one final Mapper that is easy to interpret, but is restricted to the use of hypercube covers, which is just one option of myriad potential covering schemes. 

Kalyanaraman, Kamruzzaman, and Krishnamoorthy developed methods for tracking populations within the Mapper graph by identifying \emph{interesting paths} \cite{KaKaKr2019} and \emph{interesting flares} \cite{KaKaKr2018}.
Interesting paths maximize an interestingness score, and are manifested in the Mapper graph as long paths that track particular populations that show trending behavior.
Flares capture subpopulations that diverge, i.e., show branching behavior.
In our context, we are interested in shorter paths, under the assumption that they provide the most succinct explanations for relationships between subpopulations.

Our work is similar to that of Parthasarathy, Sivakoff, Tian, and Wang \cite{parthasarathy2017quest} in that they use the Jaccard Index of an observed graph to estimate the geodesic distance of the underlying graph.
We take an approach more akin to persistence and make far fewer assumptions about properties of the underlying data.
We are unable to make rigorous estimates of distances but provide many possible representative paths.

The notion of S-paths \cite{purvine2018reprentational} is similar to stable paths if we model covers as hypergraphs, and vice versa.
Stable paths incorporate the size of each cover element (or hyperedges), normalizing the weights by relative size.
This perspective allows us to compare different parts of the resulting structure which may have wildly difference sizes of covers.
In this context, a large overlap of small elements is considered more meaningful than a proportionally small intersection of large elements. 

In Section \ref{ssec:recsys}, we show how the Steinhaus filtration and stable paths can be applied in the context of recommendation systems. 
Our viewpoint on recommendation systems is similar to work of graph-based recommendation systems. 
This is an active area of research and we believe our new perspective of interpreting such systems as covers and filtrations will yield useful tools for advancing the field. 
The general approach of graph-based recommendation systems is to model the data as a bipartite graph, with one set of nodes representing the recommendation items and the other set representing the recommendees. 
We can interpret a bipartite graph as a cover, either with elements being the recommendees covering the items, or elements being the items covering the recommendees.

%% file: sections/definitions.tex
\label{sec:definition}    


We define the Steinhaus distance, a generalization of the Jaccard distance between two sets, and further generalize it to an arbitrary collection of subsets of a cover.
These generalizations take a measure $\mu$ and assume that all sets are taken mod differences by sets of measure $0$.

\begin{definition}[Steinhaus Distance \cite{Marczewski1958}]
    The \emph{Steinhaus distance} between sets $A, B$ given a measure $\mu$ is
    $$ d_{St}(A, B) = 1 - \frac{\mu(A \cap B)}{\mu(A \cup B)} = \frac{\mu(A \cup B ) - \mu(A \cap B)}{\mu(A\cup B)} \, . $$
\end{definition}
This distance lies in $[0,1]$, is $0$ when sets are equal and $1$ when they do not intersect.


\begin{definition}[Generalized Steinhaus distance]
    Define the \emph{generalized Steinhaus distance} of a collection of sets $\{U_i\}$ as
    $$d_{St}(\{U_i\}) = 1 - \frac{\mu(\bigcap U_i )}{\mu(\bigcup U_i)} \, .$$
\end{definition}


We make use of this generalized distance to associate birth times to simplices in a nerve. 
Given a cover, we define the \Sf \ as the filtration induced from sublevel sets of the generalized Steinhaus distance function. 
In other words, consider a cover of the space and the nerve of this cover.  
For each simplex in the nerve, we assign as birth time the value of its Steinhaus distance of its corresponding cover elements. 
This filtration captures information about similarity of cover elements and overall structure of the cover.

\begin{definition}[Nerve]
    A \emph{nerve} of a cover $\Ucal=\{U_i\}_{i \in C}$ is an abstract simplicial complex defined such that 
    each subset $\{U_j\}_{j \in J} \subseteq \Ucal$ with $J \subseteq C$,
    defines a simplex if $\bigcap_{j \in J} \{U_j\} \ne \emptyset$. 
    In this construction, each cover element $U_i \in \Ucal$ defines a vertex.
\end{definition}

\begin{definition}[Steinhaus Nerve]
    \label{def:SteinhausNerve}
    The \emph{Steinhaus nerve} of a cover $\Ucal$, denoted $\Nrv_{St}(\Ucal)$, is defined as the nerve of $\Ucal$ with each simplex assigned their generalized Steinhaus distance as weight:
    $$w_\sigma = d_{St}(\{U_i \mid i \in \sigma \})~~\forall \, \sigma \in \Ucal.$$
\end{definition}

Note that $w_{\sigma} < 1$ by definition for every simplex $\sigma \in \Ucal$.
We will use $\Nrv_{St}$ when the cover $\Ucal$ is evident from context.
The weighting scheme satisfies the conditions of a filtration. 

\begin{theorem}
The Steinhaus nerve of a cover $\Ucal$ is a filtered simplicial complex.
\end{theorem}

\begin{proof}
    This proof makes use of standard set theory results.
    Let $\Ucal$ be an arbitrary cover of some set $X$ and let $\Nrv_{St}$ be its Steinhaus nerve.
    We consider $\Nrv_{St}$ as a filtration by assigning as the birth time of simplex $\sigma \in \Nrv_{St}$ its weight $w_{\sigma}$.
    To show this is indeed a filtration, we focus on a single simplex $\sigma$ and a face $\tau \preceq \sigma$ to show that the face always appears in the filtration before the simplex. 

    Suppose $\sigma$ is generated from cover elements $\{U_i\}_{i \in I}$ over some index set $I$.
    Let a face $\tau \preceq \sigma$ be generated by cover elements indexed by a subset $J \subset I$.
    The birth time of $\tau$ is
    $$d_{St}(\{U_i\}_{i \in J}) = 1 - \frac{\mu(\cap_{i \in J} U_i)}{\mu(\cup_{i \in J} U_i)}$$
    and the birth time of $\sigma$ is
    $$d_{St}\left(\{U_i\}_{i \in I}\right) = 1 - \frac{\mu(\cap_{i \in I} U_i)}{\mu(\cup_{i \in I} U_i)}.$$
    Clearly, with $\{U_i\}_{i \in J} \subset \{U_i\}_{i \in I}$, we have that $\mu(\cap_{i \in J} U_i) \ge \mu(\cap_{i \in I} U_i)$ and 
    $ \mu(\cup_{i \in J} U_i) \le \mu(\cup_{i \in I} U_i)$. 
    It follows then that $d_{St}(\tau) \le d_{St}(\sigma)$.
    With $K_{\alpha}$ denoting the subcomplex that includes all simplices in $\Nrv_{St}$ with birth time at most $\alpha \in [0,1)$,
    for any $\alpha, \beta \in [0,1)$ with $\alpha < \beta$, we have $K_\alpha \subseteq K_\beta$.
    Hence $\Nrv_{St}(\Ucal)$ is a monotonic filtration.
\end{proof}
Following this result, we refer to the construction as the Steinhaus \emph{filtration}.
Since the only cover filtrations we will use in this paper are Steinhaus filtrations, we will use the two terms interchangeably.

We could study an adaptation of \Sf \ to an analog of the VR complex by building a weighted clique rank filtration from the $1$-skeleton of the \Sf \ \cite{petri2013topological}.
This adaptation drastically reduces the number of intersection and union checks required for the construction. 
The \emph{weight rank clique filtration} is a way of generating a flag filtration from a weighted graph \cite{petri2013topological}. 
We can apply this technique to build a VR analog of the \Sf. 

\paragraph*{Note on Complexity}
\label{sec:complexity}

The complexity of constructing the \Sf \ is by and large inherited directly from the computational complexity of the nerve.
Given a cover $\Ucal$, the nerve could have at most $2^{|\Ucal|} - 1$ simplices and dimension at most $|\Ucal| - 1$ \cite{otter2017roadmap}.
These bounds are equivalent to the corresponding worst case bounds for VR and \v{C}ech complexes. 

The work involved for each simplex in constructing $\Nrv_{St}$ includes computing the volume of intersection and volume of union of the elements in the simplex.
The complexity of union and intersection operations is largely dependent on the type of data being used.
Let $\mathscr{C}_{\rm Unn}(\Vcal)$ and $\mathscr{C}_{\rm Int}(\Vcal)$ be the costs of computing the union and intersection, respectively, of a set of cover elements $\Vcal \subseteq \Ucal$.
In the worst case, we have to do $\mathscr{C}_{\rm Unn}(\Ucal) + \mathscr{C}_{\rm Int}(\Ucal)$ operations per simplex, leading to an overall worst case computational complexity of $( \mathscr{C}_{\rm Unn}(\Ucal) + \mathscr{C}_{\rm Int}(\Ucal) ) (2^{|\Ucal|} - 1)$.
If a hashing-based dictionary is constructed for each set in $\Vcal$, both $\mathscr{C}_{\rm Unn}(\Vcal)$ and $\mathscr{C}_{\rm Int}(\Vcal)$ will be at most linear in $|\Vcal|$ \cite{BiPaPa2007}.

%% file: sections/stability.tex
\label{sec:stability}

We provide a brief overview of persistence modules and the bottleneck distance. 
For details, see the book by Chazal, de Silva, Glisse, and Oudot \cite{ChdeSGlOu2016}.

Given a poset $T$, we define a \textit{persistence module} $\mathbb{V}$ over $T$ to be an indexed family of vector spaces $\{V_a\}_{a\in T}$ with a family of linear maps $\{v_a^b:V_a\rightarrow V_b\}$ that satisfy $v_b^c\circ v_a^b = v_a^c$ whenever $a\le b\le c$. 
The homologies of the Steinhaus filtered simplicial complex $H_*(\textup{Nrv}_{St}(\mathcal{V}))$ defines a persistence module where each space $V_a\in H_*(\textup{Nrv}_{St}(\mathcal{V}))$ is the homology of the simplicial complex consisting of simplices whose corresponding Steinhaus distance is at most $a$.
Herein, we use simplicial homology with coefficients in some field $k$.

Given two persistence modules $\mathbb{U}$ and $\mathbb{V}$ over $\R$ and $\epsilon\in\R$, a \textit{homomorphism of degree $\epsilon$} is a collection of linear maps 
\[
    \Phi = \{\phi_a:U_a\rightarrow V_{a+\epsilon}\}_{a\in\R}
\]
such that $v_{a+\epsilon}^{b+\epsilon}\circ\phi_a = \phi_b\circ u_a^b$. 
We denote the set of homomorphisms of degree $\epsilon$ by $\textup{Hom}^\epsilon(\mathbb{U},\mathbb{V})$. 
The composition of a homomorphism of degree $\epsilon$ and a homomorphism of degree $\epsilon'$ is a homomorphism of degree $\epsilon+\epsilon'$ given by the composition of their linear maps. 
For $\epsilon\ge0$, a crucial endomorphism of degree $\epsilon$ is the shift map $1_\mathbb{V}^\epsilon\in\textup{Hom}^\epsilon(\mathbb{V},\mathbb{V})$ given by the maps 
\[
    1_\mathbb{V}^\epsilon = \{v_a^{a+\epsilon}:V_a\rightarrow V_{a+\epsilon}\}_{a\in\R}.
\]

We say that two persistence modules $\mathbb{U}$ and $\mathbb{V}$ are \textit{$\epsilon$-interleaved} if there are homomorphisms $\Phi\in\textup{Hom}^\epsilon(\mathbb{U},\mathbb{V})$ and $\Psi\in\textup{Hom}^\epsilon(\mathbb{V},\mathbb{U})$ such that $\Psi\Phi=1_\mathbb{U}^{2\epsilon}$ and $\Phi\Psi=1_\mathbb{V}^{2\epsilon}$.

When a persistence module $\mathbb{V}$ is $q$-tame, i.e., $\textup{rank}(v_a^b)<\infty$ whenever $a<b$, then the persistence module decomposes into the direct sum of interval modules
\[
    \mathbb{V} = \bigoplus_{k\in K}\mathbb{I}(p_k^*,q_k^*).
\]
Each term $\mathbb{I}(p_k^*, q_k^*)$ represent the birth time and death time of a particular feature. 
This decomposition is often represented by the multiset of points $(p_k, q_k)$ in the extended plane $\overline{\R}^2$. 
This multiset is called the \textit{persistence diagram} of $\mathbb{V}$ and denoted $\textup{dgm}(\mathbb{V})$.

In typical applications, the simplicial complex is finite.
As a result, their homologies are finite and any map between their homologies will have finite rank.
In other words, in typical cases the homologies of the Steinhaus filtration will be $q$-tame.
However, if we allow for infinitely many cover elements then $H_*(\textup{Nrv}_{St}(\mathcal{V}))$ may not be $q$-tame.
\begin{lemma} \label{lem:nonqtame}
    There is a Steinhaus filtered simplicial complex in a totally bounded space whose homology persistence module is not $q$-tame.
\end{lemma}
\begin{proof}
    Consider the Mapper complex given in Figure \ref{fig:nonqtame}. 
    Each cover element is a rectangle of size
    \[
        R = 11\times14,~G = 6\times11,~\text{ and }~B = 11\times12
    \]
    with rectangular intersections of size $R\cap G=2.5\times4$ and $B\cap G=2.5\times 3$. 
    Under the Lebesgue measure, we get the Steinhaus distances of $d_{St}(R,G)=200/210$ and $d_{St}(B,G)=366/381$. 
    Hence, under our Steinhaus filtration the cycles going around the red-green-red-green diamonds in the complex are born at 200/210 and the remaining cycles are born at 366/381. 
    Since these cycles form a linearly independent set, we get that in the first homology persistence module the map from $200/210$ to $366/381$ has infinite rank.

    \begin{figure}[ht!]
        \begin{center}
            \def\svgwidth{0.7\textwidth}
            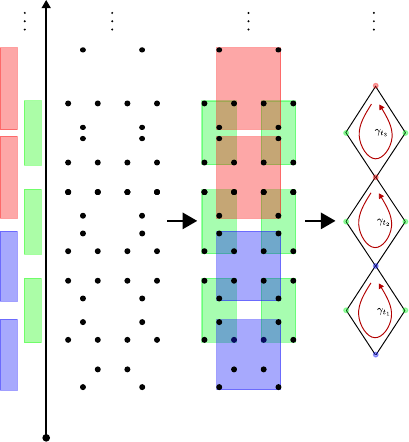
        \end{center}
        \caption{A non $q$-tame Steinhaus filtered Mapper complex.
        The center boxes follow the pattern of blue, blue, red, red, blue, blue, etc.
        For each diamond in the complex there is a cycle that traverses around the diamond.}
        \label{fig:nonqtame}
    \end{figure}

    One may modify this construction to get a non $q$-tame example in a totally bounded space by shrinking each subsequent cover by a factor of two. 
    Thus, if we make the covers have sizes
    \[
        R_n = 11\times\frac{14}{2^n},~G_n = 6\times\frac{11}{2^n},~\text{ and }~B_n=11\times\frac{12}{2^n}
    \]
    and position them such that
    \begin{align*}
        R_n\cap G_n &= 2.5\times\frac{4}{2^n} & R_n\cap G_{n+1}&=2.5\times\frac{4}{2^{n+1}} \\
        B_n\cap G_n &= 2.5\times\frac{3}{2^n} & B_n\cap G_{n+1} &= 2.5\times\frac{3}{2^{n+1}}
    \end{align*}
    then the covers are contained in a finite rectangle and the map in the first homology persistence module from $354/364$ to $630/645$ will have infinite rank.
\end{proof}


\begin{definition}
    A \textit{partial matching} between multisets $A$ and $B$ is a collection of pairs $M\subseteq A\times B$ such that for every $a\in A$, there is at most one $b\in B$ such that $(a,b)\in M$ and for every $b\in B$, there is at most one $a\in A$ such that $(a,b)\in M$.
\end{definition}

Let $\Delta\subset\overline{\R}^2$ be the diagonal. 
The cost of a partial matching $M$ is given by
\[
    \textup{cost}(M) := \max\left\{\max_{a\in A}\delta(a),\max_{b\in B}\delta(b)\right\}, \\ 
\]
where
\[
    \delta(p) = \begin{cases}
        \|p-p'\|_\infty & \text{if there is some $p'$ such that }(p,p')\in M\text{ or }(p',p)\in M, \text{ and } \\
        \inf_{q\in\Delta}\|p-q\|_\infty & \text{otherwise.}
    \end{cases}
\]

\begin{definition}[Bottleneck distance]
    The bottleneck distance between $A$ and $B$ is given by 
    \[
        d_b(A,B) = \inf\{\textup{cost}(M):M\text{ is a partial matching between $A$ and $B$}\}.
    \]
\end{definition}


\begin{theorem}[\cite{ChdeSGlOu2016}]\label{thm:bottleneck_dist_bound}
    If $\mathbb{U}$ is a $q$-tame persistence module then it has a well-defined persistence diagram $\textup{dgm}(\mathbb{U})$. 
If $\mathbb{U},\mathbb{V}$ are $q$-tame persistence modules that are $\epsilon$-interleaved then\\ $d_b(\textup{dgm}(\mathbb{U}),\textup{dgm}(\mathbb{V}))\le\epsilon$.
\end{theorem}

We define a multivalued map $C$ between $X$ and $Y$, denoted $C:X\rightrightarrows Y$, to be a subset of $X\times Y$ such that the projection onto $X$ is surjective. 
A multivalued map $C:X\rightrightarrows Y$ is said to be a \textit{correspondence} if the projection onto $Y$ is surjective, or equivalently,
\[
    C^\top := \{(y,x):(x,y)\in C\} \text{ is a multivalued map.}
\]

\begin{definition}[$\epsilon$-simplicial \cite{ChdeSOu2014}]
    Let $\mathbb{S}=\{S_\alpha\}$ and $\mathbb{T}=\{T_\beta\}$ be filtered simplicial complexes with vertex sets $X$ and $Y$. 
A multivalued map $C:X\rightrightarrows Y$ is $\epsilon$-simplicial if for any $\alpha\in\R$ and any simplex $\sigma\in S_\alpha$, every finite subset of $C(\sigma)$ is a simplex of $T_{\alpha+\epsilon}$.
\end{definition}


\begin{proposition}[\cite{ChdeSOu2014}]
    Let $\mathbb{S}$ and $\mathbb{T}$ be filtered simplicial complexes with vertex sets $X$ and $Y$. 
If $C:X\rightrightarrows Y$ is a correspondence such that $C$ and $C^\top$ are both $\epsilon$-simplicial, then they induce an $\epsilon$-interleaving between $H_*(\mathbb{S})$ and $H_*(\mathbb{T})$ given by $H_*(C)$ and $H_*(C^\top)$.
\end{proposition}

Let $\textup{Nrv}_{St}(\mathcal{U})$ and $\textup{Nrv}_{St}(\mathcal{V})$ be two Steinhaus filtered simplicial complexes. 
Given a correspondence $C:\mathcal{U}\rightrightarrows\mathcal{V}$, we define the \textit{distortion} of $C$ to be
\[
    \textup{dis}(C) := \sup\{|d_{St}(\{U_i\mid i\in\sigma\}) - d_{St}(\{V_i\mid i\in\tau\})|:(\sigma,\tau)\in C\},
\]
i.e., the largest change in birth times across each simplex.
We then define the pseudometric by finding the correspondence with the smallest distortion:
\[
    d(\mathcal{U},\mathcal{V}) = \inf\{\textup{dis}(C)\mid C:\mathcal{U}\rightrightarrows\mathcal{V}\text{ is a correspondence}\}.
\]

\begin{lemma}
    $d(\mathcal{U},\mathcal{V})$ is a pseudometric.
\end{lemma}
\begin{proof}~
    \begin{enumerate}
        \item The identity correspondence $1_\mathcal{U}=\{(U,U):U\in\mathcal{U}\}$ has distortion 0. 
        Hence, $d(\mathcal{U},\mathcal{U})=0$.

        \item Let $C:\mathcal{U}\rightrightarrows\mathcal{V}$ be a correspondence. 
        Then $\textup{dis}(C)=\textup{dis}(C^\top)$. 
        Hence, $d(\mathcal{U},\mathcal{V})=d(\mathcal{V},\mathcal{U})$.

        \item Let $C:\mathcal{U}\rightrightarrows\mathcal{V}$ and $D:\mathcal{V}\rightrightarrows\mathcal{W}$ be correspondences.
        Then 
        \[
            D\circ C = \{(x,z):\exists y\in\mathcal{V}\text{ such that }(x,y)\in C\text{ and }(y,z)\in D\}
        \]
        is a correspondence from $\mathcal{U}$ to $\mathcal{W}$. 
        For each pair $(\sigma,\tau)\in D\circ C$ there is some simplex $\omega$ such that $\omega\in C(\sigma)$ and $\tau\in D(\omega)$. 
        Note that,
        \begin{align*}
            |d_{St}(\{U_i\mid i\in\sigma\}) - d_{St}(\{W_i\mid i\in\tau\})| &\le |d_{St}(\{U_i\mid i\in\sigma\}) - d_{St}(\{V_i\mid i\in\omega\})| \\
            &\phantom{\le} + |d_{St}(\{V_i\mid i\in\omega\}) - d_{St}(\{W_i\mid i\in\tau\})|,
        \end{align*}
    \end{enumerate}
    \hspace*{0.2in} i.e., $\textup{dis}(D\circ C)\le\textup{dis}(C) + \textup{dis}(D)$. 
    Hence, $d(\mathcal{U},\mathcal{W}) \le d(\mathcal{U},\mathcal{V}) + d(\mathcal{V},\mathcal{W})$.
\end{proof}

$d(\mathcal{U},\mathcal{V})$ does not satisfy the positivity requirement to be a metric. 
Consider the covers $\mathcal{U}$ and $\mathcal{V}$ given in Figure \ref{fig:positiviity_ce}.
The simplicial complexes are not isomorphic, but $d(\mathcal{U},\mathcal{V})=0$.

\begin{figure}[ht]
    \begin{center}
        \input{images/positivity_ce.tex}
    \end{center}
    \caption{Nonisomorphic simplicial complexes for covers $\mathcal{U}$ (top row) and $\mathcal{V}$ (bottom row) with $d(\mathcal{U},\mathcal{V})=0$.}
    \label{fig:positiviity_ce}
\end{figure}


\addtocounter{theorem}{1}

\begin{proposition}\label{prop:interleaving_result}
    Let $\textup{Nrv}_{St}(\mathcal{U})$ and $\textup{Nrv}_{St}(\mathcal{V})$ be two Steinhaus filtered simplicial complexes. 
If $\epsilon>d(\mathcal{U},\mathcal{V})$, then the persistence modules $H_*(\textup{Nrv}_{St}(\mathcal{U}))$ and $H_*(\textup{Nrv}_{St}(\mathcal{V}))$ are $\epsilon$-interleaved.
\end{proposition}
\begin{proof}
    If $\epsilon>d(\mathcal{U},\mathcal{V})$ then there is a correspondence $C:\mathcal{U}\rightrightarrows\mathcal{V}$ such that $\textup{dis}(C)\le\epsilon$. 
    Notice that for any simplex $\sigma\in\textup{Nrv}_{St}(\mathcal{U})$ with birth time at most $\alpha$ and simplex $\tau\subseteq C(\sigma)$,
    \[
        |d_{St}(\{U_i\mid i\in\sigma\}) - d_{St}(\{V_i\mid i\in\tau\})|\le\epsilon.
    \]
    By the reverse-triangle inequality we get that
    \[
        d_{St}(\{V_i\mid i\in\tau\}) \le d_{St}(\{U_i\mid i\in\sigma\}) + \epsilon \le \alpha + \epsilon,
    \]
    i.e., $C$ is $\epsilon$-simplicial. 
    An identical argument shows that $C^\top$ is also $\epsilon$-simplicial.
\end{proof}

\begin{theorem}
    \label{thm:stability}
    Let $\textup{Nrv}_{St}(\mathcal{U})$ and $\textup{Nrv}_{St}(\mathcal{V})$ be two Steinhaus filtered simplicial complexes with $q$-tame homology persistence modules. 
    Then
    \[
        d_b(\textup{dgm}(H_*(\textup{Nrv}_{St}(\mathcal{U}))),\textup{dgm}(H_*(\textup{Nrv}_{St}(\mathcal{V})))) \le d(\mathcal{U},\mathcal{V}).
    \]
\end{theorem}
\begin{proof}
    Follows from Proposition \ref{prop:interleaving_result} and Theorem \ref{thm:bottleneck_dist_bound}.
\end{proof}

%% file: images/non_qtame_example.pdf_tex
\begingroup%
  \makeatletter%
  \providecommand\color[2][]{%
    \errmessage{(Inkscape) Color is used for the text in Inkscape, but the package 'color.sty' is not loaded}%
    \renewcommand\color[2][]{}%
  }%
  \providecommand\transparent[1]{%
    \errmessage{(Inkscape) Transparency is used (non-zero) for the text in Inkscape, but the package 'transparent.sty' is not loaded}%
    \renewcommand\transparent[1]{}%
  }%
  \providecommand\rotatebox[2]{#2}%
  \newcommand*\fsize{\dimexpr\f@size pt\relax}%
  \newcommand*\lineheight[1]{\fontsize{\fsize}{#1\fsize}\selectfont}%
  \ifx\svgwidth\undefined%
    \setlength{\unitlength}{195.34251536bp}%
    \ifx\svgscale\undefined%
      \relax%
    \else%
      \setlength{\unitlength}{\unitlength * \real{\svgscale}}%
    \fi%
  \else%
    \setlength{\unitlength}{\svgwidth}%
  \fi%
  \global\let\svgwidth\undefined%
  \global\let\svgscale\undefined%
  \makeatother%
  \begin{picture}(1,1.08439547)%
    \lineheight{1}%
    \setlength\tabcolsep{0pt}%
    \put(0,0){\includegraphics[width=\unitlength,page=1]{non_qtame_example.pdf}}%
  \end{picture}%
\endgroup%

%% file: images/positivity_ce.tex
\begin{tikzpicture}[scale=1.5]
    \pgfmathsetmacro\ya{0}
    \pgfmathsetmacro\yb{0.866}
    \draw[fill=black] (0,\ya) circle (1pt);
    \draw[fill=black] (0.5,\yb) circle (1pt);
    \draw[fill=black] (1,\ya) circle (1pt);
    \draw[fill=black] (1.5,\yb) circle (1pt);
    \draw[fill=black] (2,\ya) circle (1pt);
    \draw[fill=black] (4,\ya) circle (1pt);
    \draw[fill=black] (4.5,\yb) circle (1pt);
    \draw[fill=black] (5,\ya) circle (1pt);
    \draw[fill=black] (5.5,\yb) circle (1pt);
    \draw[fill=black] (6,\ya) circle (1pt);

    \draw (0,\ya) -- (0.5,\yb) -- (1,\ya) -- cycle;
    \draw (1,\ya) -- (1.5,\yb) -- (2,\ya) -- cycle;
    \draw (4,\ya) -- (4.5,\yb) -- (5,\ya) -- cycle;
    \draw (5,\ya) -- (5.5,\yb) -- (6,\ya) -- cycle;
    \draw (0.5,\yb) -- (1.5,\yb);
    \draw (4.5,\yb) -- (5.5,\yb);

    \draw[fill=red,opacity=0.25] (0,\ya) circle (0.52);
    \draw[fill=red,opacity=0.25] (0.5,\yb) circle (0.52);
    \draw[fill=red,opacity=0.25] (1,\ya) circle (0.52);
    \draw[fill=red,opacity=0.25] (1.5,\yb) circle (0.52);
    \draw[fill=red,opacity=0.25] (2,\ya) circle (0.52);
    \draw[fill=red,opacity=0.25] (4,\ya) circle (0.52);
    \draw[fill=red,opacity=0.25] (4.5,\yb) circle (0.52);
    \draw[fill=red,opacity=0.25] (5,\ya) circle (0.52);
    \draw[fill=red,opacity=0.25] (5.5,\yb) circle (0.52);
    \draw[fill=red,opacity=0.25] (6,\ya) circle (0.52);

    \draw (0,\ya) circle (0.52);
    \draw (0.5,\yb) circle (0.52);
    \draw (1,\ya) circle (0.52);
    \draw (1.5,\yb) circle (0.52);
    \draw (2,\ya) circle (0.52);
    \draw (4,\ya) circle (0.52);
    \draw (4.5,\yb) circle (0.52);
    \draw (5,\ya) circle (0.52);
    \draw (5.5,\yb) circle (0.52);
    \draw (6,\ya) circle (0.52);

    \pgfmathsetmacro\ya{-2.5}
    \pgfmathsetmacro\yb{-1.634}
    \draw[fill=black] (0,\ya) circle (1pt);
    \draw[fill=black] (0.5,\yb) circle (1pt);
    \draw[fill=black] (1,\ya) circle (1pt);
    \draw[fill=black] (3,\ya) circle (1pt);
    \draw[fill=black] (3.5,\yb) circle (1pt);
    \draw[fill=black] (4,\ya) circle (1pt);
    \draw[fill=black] (4.5,\yb) circle (1pt);
    \draw[fill=black] (5,\ya) circle (1pt);
    \draw[fill=black] (5.5,\yb) circle (1pt);
    \draw[fill=black] (6,\ya) circle (1pt);

    \draw (0,\ya) -- (0.5,\yb) -- (1,\ya) -- cycle;
    \draw (3,\ya) -- (3.5,\yb) -- (4,\ya) -- cycle;
    \draw (4,\ya) -- (4.5,\yb) -- (5,\ya) -- cycle;
    \draw (5,\ya) -- (5.5,\yb) -- (6,\ya) -- cycle;
    \draw (3.5,\yb) -- (4.5,\yb) -- (5.5,\yb);

    \draw[fill=red,opacity=0.25] (0,\ya) circle (0.52);
    \draw[fill=red,opacity=0.25] (0.5,\yb) circle (0.52);
    \draw[fill=red,opacity=0.25] (1,\ya) circle (0.52);
    \draw[fill=red,opacity=0.25] (3,\ya) circle (0.52);
    \draw[fill=red,opacity=0.25] (3.5,\yb) circle (0.52);
    \draw[fill=red,opacity=0.25] (4,\ya) circle (0.52);
    \draw[fill=red,opacity=0.25] (4.5,\yb) circle (0.52);
    \draw[fill=red,opacity=0.25] (5,\ya) circle (0.52);
    \draw[fill=red,opacity=0.25] (5.5,\yb) circle (0.52);
    \draw[fill=red,opacity=0.25] (6,\ya) circle (0.52);

    \draw (0,\ya) circle (0.52);
    \draw (0.5,\yb) circle (0.52);
    \draw (1,\ya) circle (0.52);
    \draw (3,\ya) circle (0.52);
    \draw (3.5,\yb) circle (0.52);
    \draw (4,\ya) circle (0.52);
    \draw (4.5,\yb) circle (0.52);
    \draw (5,\ya) circle (0.52);
    \draw (5.5,\yb) circle (0.52);
    \draw (6,\ya) circle (0.52);
\end{tikzpicture}

%% file: sections/equivalence.tex
\label{sec:equivalence}

To situate the Steinhaus filtration in the context of standard filtrations built on point clouds in Euclidean space, we wish to show the equivalence of the Steinhaus filtration to the \v{C}ech and VR filtrations under certain conditions. 
We first show that the \v{C}ech filtration on a finite set of points in dimension one, i.e., the nerve of balls with radius $r$ around each point and over a sequence of $r$, and the Steinhaus Nerve constructed from the terminal cover of the \v{C}ech filtration are isomorphic (see Theorem \ref{thm:cecheqv1d}).
Precisely,
there exists a continuous bijection between insertion times of the Steinhaus Nerve and insertion times of the \v{C}ech filtration.

We provide experimental evidence for the $1$-skeletons of the Steinhaus filtration and the VR filtration being isomorphic.
We prove one direction of this equivalence in arbitrary dimension---that the VR filtration completely determines the $1$-skeleton of the Steinhaus filtration (see Lemma \ref{lem:vrdetstnhs}).
We prove by example that the Steinhaus filtration and the \v{C}ech filtration in two (or higher) dimension(s) of a finite set of points may not be isomorphic (see Proposition \ref{prop:cech_cover_not_iso}).

Let $\check{C}_r(X)$ be the cover of $X$ by balls of radius $r$ centered on points in $X$.
The \v{C}ech complex is the nerve of this cover.
The \v{C}ech filtration is a sequence of simplicial complexes for all $r$.


\begin{theorem} \label{thm:cecheqv1d}
    Given a finite dataset $X\subset\R$ in dimension one and some radius $R>\diam(X)$ the \v{C}ech filtration constructed from $X$ is isomorphic to the Steinhaus filtration on $X$ constructed from $\check{C}_R(X)$, given the Lebesgue (i.e., volume) measure.
\end{theorem}

\begin{proof}
%


We set $\check{C}( \{ v_i \} )$ as the birth radius of the simplex defined by the set $\{v_i\}$, computed as 
\[ \check{C}(\{v_i\})=\frac{\text{max}_i(v_i)-\text{min}_i(v_i)}{2} \]
since in 1D space the associated simplex is born when balls around the two outermost points intersect.

Let $\{V_i\}$ be the set of balls of radius $R$ centered on the set $\{v_i\}$. Recall that we are using the Steinhaus distance for Lebesgue measure, so $\mu$ computes volume here. Then the generalized Steinhaus distance for those balls is given by $d_{St}(\{V_i\})=1-\frac{\text{min}(v_i+R)-\text{max}(v_i-R)}{\text{max}(v_i+R)-\text{min}(v_i-R)}$, since the mutual intersection of all the balls in this one dimensional space is the interval bounded by the minimum right endpoint of all the balls and the maximum of all left endpoints of all the balls and the union of all the balls has the minimum left endpoint amongst left endpoints and maximum right endpoint amongst right endpoints.
  \begin{align*}
  d_{St}(\{V_i\}) & = 1-\frac{\text{min}(v_i)-\text{max}(v_i)+2R}{\text{max}(v_i)-\text{min}(v_i)+2R} \\
  & = 1-\frac{-2\check{C}(\{v_i\})+2R}{2\check{C}(\{v_i\})+2R} 
  = 1-\frac{R-\check{C}(\{v_i\})}{R+\check{C}(\{v_i\})}.
  \end{align*}    
Solving for $\check{C}(\{v_i\})$, we get 
\[
    \check{C}(\{v_i\})=\frac{Rd_{St}(\{V_i\})}{2-d_{St}(\{V_i\})},
    ~\text{ establishing a bijection between birth times.}
\]

Now since $R>\diam(X)$ and $\{v_i\}\subseteq X$, $0\leq \check{C}(\{v_i\})\leq R$.
Also, $\frac{R-x}{R+x}$ decreases monotonically over the range $x\in [0,R]$.
Thus, $1-\frac{R-x}{R+x}$ increases monotonically on $x\in [0,R]$.
Thus, if we order the subsets of $X$ by increasing birth radius $(s_1,\dots,s_n)$, then the Steinhaus distances $(d_{St}(s_1),$ $\dots,d_{St}(s_n))$ are also in increasing order. Therefore, the two filtrations are isomorphic.
\end{proof}


\begin{lemma} \label{lem:vrdetstnhs}
    There is a map from the birth time of an edge in the VR filtration to its birth time in the Steinhaus filtration.
    Hence the VR filtration determines the $1$-skeleton of the Steinhaus filtration in arbitrary dimensions. 
\end{lemma}

\begin{proof}
The volume of the intersection of two hyperspheres was derived by Li \cite{li2011concise}.
The volume of intersection of two hyperspheres of equal radius $R$ in $\R^n$ with centers distance $d$ apart is defined as 
$$
V^n_\cap(R, d) = \frac{\pi^{n / 2}}{\Gamma (\frac{n}{2} + 1)}R^n I_{1 - (d/2) / R^2}\left(\frac{n+1}{2}, \frac{1}{2}\right)
$$
where $\Gamma$ is the gamma function and $I$ is the regularized incomplete beta function:
$$ I_z(a,b) = \frac{ \Gamma(a + b) \int_0^z u^{a-1}(1-u)^{b-1} du } {\Gamma(a)\Gamma(b)} \, .$$

We can reduce this equation to 
\begin{align*}
V^n_\cap(R, d) 
&= R^n \pi^{ (n-1) / 2}  \frac{ \int_0^{1 - (d/ 2) / R^2} u^{(n-1 )/ 2}(1-u)^{-1/2} du } {\Gamma(\frac{n+1}{2})} \, .
\end{align*}
The volume of an $n$-sphere of radius $R$ is 
$$V_{\circ}(R) = \frac{\pi ^{n/2}}{\Gamma(\frac{n}{2} + 1)}R^n $$
and so the volume of union of two $n$-spheres is 
$$ V^n_{\cup}(R, d) = 2 V^n_{\circ}(R) - V^n_\cap(R,d) \,. $$

We then compute the Steinhaus distance with Lebesgue measure of two spheres in $\R^n$ and radius $R$ with Euclidean distance $d$ apart as 
\begin{align*}
    d^n_{St}(R, d) 
    &=  \frac{2 V^n_{\circ}(R) - 2 V^n_\cap(R,d)}{2 V^n_{\circ}(R) - V^n_\cap(R,d)} \\
    &=  \frac{
        2\frac{\pi ^{n/2}}{\Gamma(\frac{n}{2} + 1)}R^n - 2R^n \pi^{ (n-1) / 2}  \frac{ \int_0^{1 - (d^2/ 2d) / R^2} u^{(n-1 )/ 2}(1-u)^{-1/2} du } {\Gamma(\frac{n+1}{2})}
    }{
        2\frac{\pi ^{n/2}}{\Gamma(\frac{n}{2} + 1)}R^n - R^n \pi^{ (n-1) / 2}  \frac{ \int_0^{1 - (d^2/ 2d) / R^2} u^{(n-1 )/ 2}(1-u)^{-1/2} du } {\Gamma(\frac{n+1}{2})}
    }\\
    &=  \frac{  
            2 \Gamma(\frac{n+1}{2})
            -  n\Gamma(\frac{n}{2})  \pi^{-1/2} \int_0^{1 - (d/ 2) / R^2} u^{(n-1 )/ 2}(1-u)^{-1/2} du 
        }
        { 
            2\Gamma(\frac{n+1}{2}) 
            - \frac{n}{2}\Gamma(\frac{n}{2}) \pi^{-1/2}  \int_0^{1 - (d/ 2) / R^2} u^{(n-1 )/ 2}(1-u)^{-1/2} du
        } \, .
\end{align*}

This equation provides a mapping from the birth time of the edge in the VR filtration to the birth time of the edge in the Steinhaus filtration.
Once an $n$ and $R$ are chosen, the equation readily reduces, producing the birth times of a simplex in the Steinhaus filtration. 
\end{proof}


Next we provide detailed experimental results suggesting that the $1$-skeletons of the Steinhaus and the VR filtrations are isomorphic.
We estimate the area of intersection of 1-spheres using Monte Carlo integration with uniform sampling.
The first plot in Figure \ref{fig:comparison} shows the 50 landmark points along with 20,000 points uniformly sampled around the landmarks for ease of computation.
The middle plot shows the persistence diagrams of dimension 0 and 1 for the VR filtration on the landmarks.
Finally, we show the Steinhaus filtration on the landmarks using balls with radii 0.5 as the covers.
\begin{figure}[htp!]
    \begin{center}
        \includegraphics[width=1\textwidth]{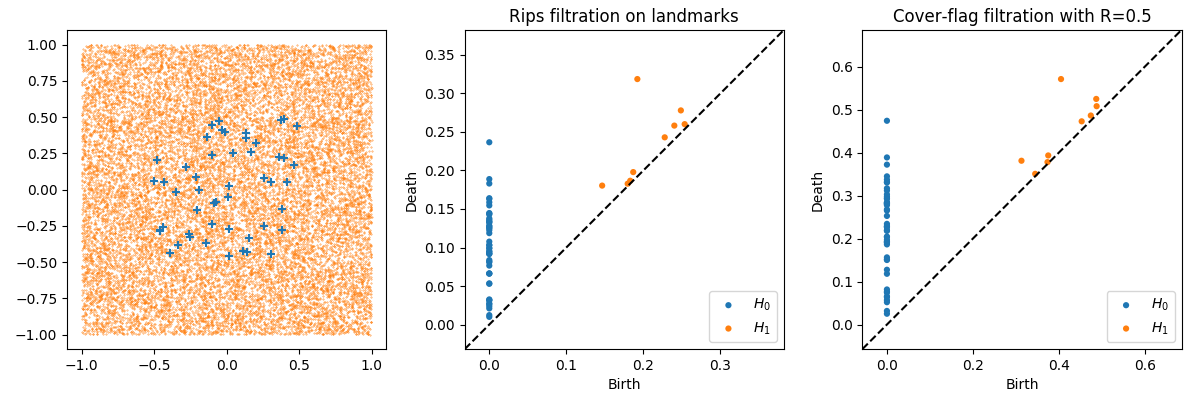}
    \end{center}
    \vspace*{-0.1in}        
    \caption{
    Persistence diagrams for the Vietoris-Rips filtration and the approximate Steinhaus filtration of a set of uniformly sampled points in the plane.}
    \label{fig:comparison}
\end{figure}

We approximate the Steinhaus filtration similar to how the VR approximates \v{C}ech filtration, i.e., by only computing the $1$-skeleton of the nerve, and including any higher order simplices for which all faces are already contained in the filtration, taking the maximum birth time of all faces.
We note that the two persistence diagrams have only minor differences (only in dimension 1).


\begin{proposition} 
    \label{prop:cech_cover_not_iso}
    The \v{C}ech filtration and the Steinhaus filtration are not isomorphic in general.
\end{proposition}
\begin{proof}
Let $X\subset\R^2$ be the four corners of a unit square. We construct a cover where at each point in $X$ there is a ball of radius one (see Figure \ref{fig:nonequiv_example}).
    In the Steinhaus filtration, four 2-simplices are born approximately at 0.935 before the 3-simplex is born around 0.959 under the volume measure.
    However, in the \v{C}ech complex the 2-simplices are born at the same time as the 3-simplex.
\end{proof}

\begin{figure}[ht!]
    \begin{center}
        \input{images/nonequiv.tex}
    \end{center}
    \caption{Example showing that the \v{C}ech filtration and Steinhaus filtration are not isomorphic in general. The \v{C}ech filtration (middle) goes from 1-simplices directly to a 3-simplex whereas the Steinhaus filtration (bottom) of the cover (top) has four 2-simplices born before the 3-simplex.}
    \label{fig:nonequiv_example}
\end{figure}

%% file: images/nonequiv.tex
\begin{tikzpicture}[scale=1.5]
    \draw[fill=black] (0,0) circle (1pt);
    \draw[fill=black] (1,0) circle (1pt);
    \draw[fill=black] (0,1) circle (1pt);
    \draw[fill=black] (1,1) circle (1pt);

    \draw[fill=red,opacity=0.25] (0,0) circle (1);
    \draw[fill=red,opacity=0.25] (1,0) circle (1);
    \draw[fill=red,opacity=0.25] (0,1) circle (1);
    \draw[fill=red,opacity=0.25] (1,1) circle (1);

    \draw (0,0) circle (1);
    \draw (1,0) circle (1);
    \draw (0,1) circle (1);
    \draw (1,1) circle (1);

    \node (cech) at (0.5,-1.5) {\v{C}ech Filtration:};
    \draw[fill=black] (-2,-3) circle (1pt);
    \draw[fill=black] (-1,-3) circle (1pt);
    \draw[fill=black] (-2,-2) circle (1pt);
    \draw[fill=black] (-1,-2) circle (1pt);
    \node (cr0) at (-1.5,-3.5) {$r=0$};

    \draw[fill=black] (0,-3) circle (1pt);
    \draw[fill=black] (1,-3) circle (1pt);
    \draw[fill=black] (0,-2) circle (1pt);
    \draw[fill=black] (1,-2) circle (1pt);
    \draw (0,-3) -- (0,-2) -- (1,-2) -- (1,-3) -- cycle;
    \node (cr2) at (0.5,-3.5) {$r=\frac{1}{2}$};

    \draw[red!60,fill=red!60] (2,-3,-1) -- (3,-3,0) -- (2,-2,0) -- cycle;
    \draw[red!60,fill=red!60] (3,-2,0) -- (3,-3,0) --(2,-2,0)--cycle;
    \draw[red!60,fill=red!60] (2,-3,-1) -- (3,-3,0) --(3,-2,0)--cycle;

    \draw[fill=black] (2,-3,-1) circle (1pt);
    \draw[fill=black] (3,-3,0) circle (1pt);
    \draw[fill=black] (2,-2,0) circle (1pt);
    \draw[fill=black] (3,-2,0) circle (1pt);
    \draw (2,-3,-1) -- (3,-3,0) -- (2,-2,0) -- cycle;
    \draw (3,-2, 0) -- (3,-3,0) -- (2,-2,0) -- cycle;
    \draw[dashed] (2,-3,-1) -- (3,-2,0);
    \node (cr3) at (2.5,-3.5) {$r=\frac{\sqrt{2}}{2}$};

    \node (steinhaus) at (0.5,-4.5) {Steinhaus Filtration:};
    \draw[fill=black] (-4,-5) circle (1pt);
    \draw[fill=black] (-3,-5) circle (1pt);
    \draw[fill=black] (-3,-6) circle (1pt);
    \draw[fill=black] (-4,-6) circle (1pt);
    \node (st0) at (-3.5,-6.5) {$d_{St}=0$};

    \draw[fill=black] (-2,-5) circle (1pt);
    \draw[fill=black] (-1,-5) circle (1pt);
    \draw[fill=black] (-1,-6) circle (1pt);
    \draw[fill=black] (-2,-6) circle (1pt);
    \draw (-2,-5) -- (-1,-5) -- (-1,-6) -- (-2,-6) -- cycle;
    \node (st2) at (-1.5,-6.5) {$d_{St}\approx0.756$};

    \draw[fill=black] (0,-5) circle (1pt);
    \draw[fill=black] (1,-5) circle (1pt);
    \draw[fill=black] (1,-6) circle (1pt);
    \draw[fill=black] (0,-6) circle (1pt);
    \draw (0,-5) -- (1,-5) -- (1,-6) -- (0,-6) -- cycle;
    \draw (0,-5) -- (1,-6);
    \draw (1,-5) -- (0,-6);
    \node (st22) at (0.5,-6.5) {$d_{St}\approx0.9$};

    \draw[red!60,fill=red!60,opacity=0.5] (2,-6,-1) -- (3,-6,0) -- (2,-5,0) -- cycle;
    \draw[red!60,fill=red!60,opacity=0.5] (3,-5,0) -- (3,-6,0) --(2,-5,0)--cycle;

    \draw[fill=black] (2,-6,-1) circle (1pt);
    \draw[fill=black] (3,-6,0) circle (1pt);
    \draw[fill=black] (2,-5,0) circle (1pt);
    \draw[fill=black] (3,-5,0) circle (1pt);
    \draw (2,-6,-1) -- (3,-6,0) -- (2,-5,0) -- cycle;
    \draw (3,-5, 0) -- (3,-6,0) -- (2,-5,0) -- cycle;
    \draw[dashed] (2,-6,-1) -- (3,-5,0);
    \node (cr3) at (2.5,-6.5) {$d_{St}\approx0.935$};

    \draw[red!60,fill=red!60] (4,-6,-1) -- (5,-6,0) -- (4,-5,0) -- cycle;
    \draw[red!60,fill=red!60] (5,-5,0) -- (5,-6,0) --(4,-5,0)--cycle;

    \draw[fill=black] (4,-6,-1) circle (1pt);
    \draw[fill=black] (5,-6,0) circle (1pt);
    \draw[fill=black] (4,-5,0) circle (1pt);
    \draw[fill=black] (5,-5,0) circle (1pt);
    \draw (4,-6,-1) -- (5,-6,0) -- (4,-5,0) -- cycle;
    \draw (5,-5, 0) -- (5,-6,0) -- (4,-5,0) -- cycle;
    \draw[dashed] (4,-6,-1) -- (5,-5,0);
    \node (cr4) at (4.5,-6.5) {$d_{St}\approx0.959$};
\end{tikzpicture}

%% file: sections/paths.tex
\label{sec:paths}


We develop a theory of stable paths within a Steinhaus filtration. 
We provide an algorithm for computing a most stable path from one vertex to another.
Note that a most stable path might not be a shortest path in terms of number of edges.
Conversely, a shortest path might not be highly stable.
Since the two objectives are at odds with each other, we provide an algorithm to identify a family of shortest paths as we vary the stability level, akin to computing persistent homology.

We were studying shortest paths in a Mapper graph constructed on a machine learning model as ways to illustrate the relations between the data as identified by the model.  
In this context, shortest paths found could have low Steinhaus distance, and thus could be considered noise. 
This motivated our desire to find stable paths, as they would intuitively be most representative of the dataset and stable with respect to changing parameters in Mapper or changing data.

\begin{definition}[$\rho$-Stable Path]
    Given a Steinhaus distance $\rho$, a path $P$ is said to be \emph{$\rho$-stable} if 
    $$ \max \{d_{St}(e) \mid e \in P \} \leq  \rho \,  $$
    where $d_{St}(e)$ is the Steinhaus distance associated with edge $e$.
\end{definition}
It follows that a $\rho_1$-stable path is also $\rho_2$-stable for any $\rho_2 \geq \rho_1$.
Also, a $\rho_1$-stable path $P_1$ is more stable than a $\rho_2$-stable path $P_2$ when $\rho_1 < \rho_2$.
In this case, we have a higher confidence that the edges in $P_1$ do exist, and are not due to noise, than the edges in $P_2$.


\begin{definition}[Most Stable Path]
Given a pair of vertices $s$ and $t$, a \emph{most stable $s$-$t$ path} is a $\rho$-stable path between $s$ and $t$ for the smallest value of $\rho$.
If there are multiple $s$-$t$ paths at the same minimum $\rho$ value, a shortest path among them is defined as a most stable path.
\end{definition}


%
The problem of finding the most stable $s$-$t$ path can be solved as a minimax path problem on an undirected graph, which can solved efficiently using, e.g., range minimum queries \cite{DeLaWe2009}.

We are then left with two paths between vertices $s$ and $t$, the shortest and the most stable.
It should be clear that the shortest path is not necessarily stable and the stable path is not necessarily short.
As these two notions, stable and short, are at odds with each other, we are interested in computing the entire Pareto frontier between the short and stable path,
which will be a curve capturing a set of paths that represent the best trade-off between their stability and length.
Moving along the Pareto frontier, one could improve stability at the expense of length or vice versa, but one could not improve both at once.
We present an algorithm to identify the Pareto frontier in Algorithm \ref{fig:paretoaglo}, and a visualization of the output from this algorithm in Figure \ref{fig:frontier}.

\begin{algorithm}[ht!]
    \caption{Algorithm to identify the Pareto frontier between shortest and most stable paths.}
    \label{fig:paretoaglo}
    \begin{algorithmic}
        \State \textbf{Input:} 1-skeleton $G$ of Steinhaus filtration and vertices $s,t$
        \State \textbf{Output:} A list $\textup{LIST}$ consisting of pairs $[P,\rho]$
        \State set $\textup{LIST}=\emptyset$
        \While{$s,t$ are connected in $G$}
            \State find a shortest path $P$ between $s$ and $t$
            \State $\rho=\max\{d_{St}(e)\mid e\in P\}$
            \If{$\textup{LIST}$ has no pair $[P',\rho']$ with $|P|=|P'|$}
                \State add $[P,\rho]$ to $\textup{LIST}$
            \Else
                \If{$\rho<\rho'$ for $[P',\rho']\in\textup{LIST}$ with $|P|=|P'|$}
                    \State replace $[P',\rho']$ with $[P,\rho]$ in $\textup{LIST}$
                \EndIf
            \EndIf
            \State remove all edges $e$ from $G$ with $d_{St}(e)\ge\rho$
        \EndWhile
    \end{algorithmic}
\end{algorithm}

We repeatedly compute the shortest path, while essentially sweeping over the Steinhaus Distance.
This process results in a Pareto frontier balancing the shortest paths with the stability of those paths. 

\begin{figure}[ht!]
  \centering
  \includegraphics[width=0.8\textwidth]{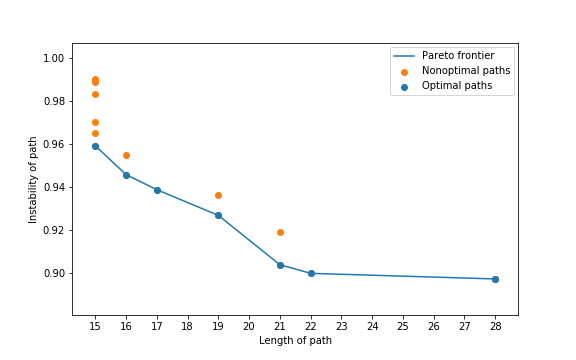}
  \caption{    \label{fig:frontier}
     Pareto frontier between length of path and stability of path.
     }
\end{figure}
The blue points in Figure \ref{fig:frontier} are on the Pareto frontier, while the orange points are the pairs $[P',\rho']$ that get replaced from the LIST in the course of the algorithm.
We then visualize the paths on the Pareto frontier in Figure \ref{fig:demo_paths}.
Continuing our analogy to persistence, the path corresponding to a point on the Pareto frontier which sees a steep rise to the left is considered highly persistent, e.g., the path with length $21$ on the frontier.



\begin{figure}[hbp!]
    \begin{center}
        \includegraphics[width=1\textwidth]{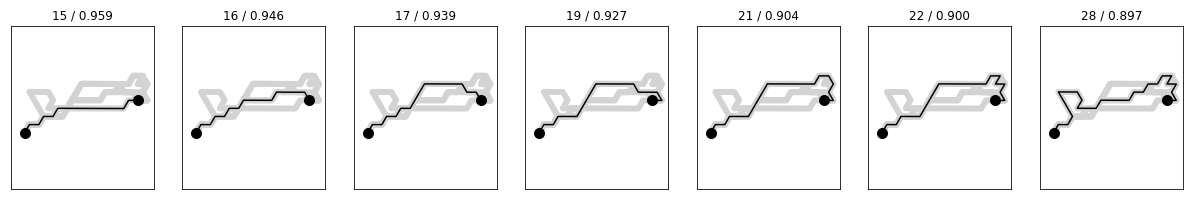}
    \end{center}
    \caption{Visualization of each path on the Pareto frontier shown in Figure \ref{fig:frontier}. }
    \label{fig:demo_paths}
\end{figure}

%% file: sections/applications.tex
We apply the Steinhaus filtration and stable paths to recommendation systems and Mapper analysis.
We show how recommendation systems can be modeled using the Steinhaus filtration and then show how stable paths in this filtration can answer the question \emph{what movies should I show my friend first, to wean them into my favorite (but potentially weird) movie?}
We then define an extension of the traditional Mapper complex \cite{SiMeCa2007,CaOu2018} called the Steinhaus Mapper Filtration, and show how stable paths within this filtration can provide valuable explanations of populations in the Mapper construction.
As a direct illustration, we focus on the case of explainable machine learning, where the Mapper complex is constructed with a supervised machine learning model as the filter function, and address the question \emph{what can we learn about the model?} 

Other potential applications include sensor networks and finding driving directions.
Sensor coverage areas are often not uniform balls, and the Steinhaus filtration is aptly suited for developing a filtration.
In the context of communication networks, stable paths could be interpreted as reliable routes.
One could seek driving directions that take not only short, but also ``easy'' routes.

%% file: sections/recommendation.tex
We apply the Steinhaus filtration to a recommendation system dataset and employ the stable paths analysis to compute sequences of movies that ease viewers from one title to another title.
Suppose you have only ever seen the movie \emph{Mulan} and your partner wants to show you \emph{Moulin Rouge}.
It would be jarring to watch that movie,
so your partner might gently build up to it by showing movies similar to both \emph{Mulan} and \emph{Moulin Rouge}.
We compute stable paths that identify such a gentle sequence. 

We use the MovieLens-20m dataset \cite{harper2016movielens},
comprised of 20 million ratings by 138,493 users of 27,278 movies.
Often, these types of datasets are interpreted as bipartite graphs. 
We note that a bipartite graph can be equivalently represented as a covering of one node set with the other, and apply the Steinhaus filtration to build a filtration.
In our case, we interpret each movie as a cover element of the users who have rated the movie.
To reduce computational expenses and noise, we remove all movies with less than 10 ratings and then sample 4000 movies at random from the remaining movies. 

\begin{figure}[ht!]
    \centering
    \includegraphics[width=0.8\textwidth]{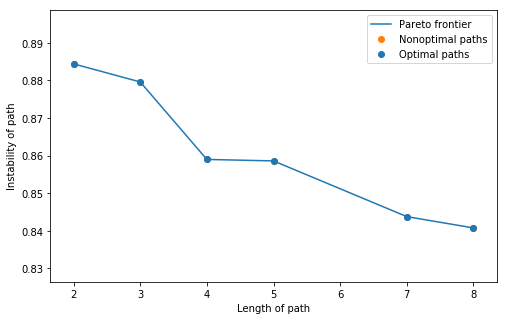}
    \caption{    \label{fig:mulan_frontier}
    Pareto frontier of stable paths between \emph{Mulan} and \emph{Moulin Rouge}.
    }
\end{figure}
Figure \ref{fig:mulan_frontier} shows the computed Pareto frontier of stable paths for the case of \emph{Mulan} and \emph{Moulin Rouge}.
In Table \ref{tab:mulan}, we show two stable paths. 
The stable path with length 4 is found after a large drop in instability.
As the length and stability must be traded off, we think this would be a decent path to choose if you want to optimize both.
The second path is the most stable.
For readers 
familiar with the movies,
the relationship between each edge is clear,
even if the path may be a bit on the longer side.

\begin{table}[htp!]
 \centering
 \vspace*{0.15in}
 \caption{Two sequences of movie transitions.  \label{tab:mulan}}
    \begin{tabular}{p{.4\textwidth}p{.4\textwidth}} \midrule
       \hspace*{0.3in} Shortest Path & \hspace*{0.3in} Most Stable Path\\\midrule
       \vspace*{-0.1in}
         \begin{enumerate}
            \item Mulan (1998)
            \item Dumbo (1941)
            \item The Sound of Music (1965)
            \item Moulin Rouge (2001)
         \end{enumerate}&
       \vspace*{-0.1in}
         \begin{enumerate}
            \item Mulan (1998)
            \item Robin Hood (1973)
            \item Dumbo (1941)
            \item The Sound of Music (1965)
            \item Gone with the Wind (1939)
            \item Psycho (1960)
            \item High Fidelity (2000)
            \item Moulin Rouge (2001)
        \end{enumerate} 
    \end{tabular}
\end{table}

%% file: sections/mapper.tex
We develop a method of model induction for inspecting a machine learning model. 
The goal is to develop an understanding of the model structure by characterizing the relationship between the feature space and the prediction space. 
The gleaned understanding can help non-experts make sense of algorithmic decisions.
The Mapper construction \cite{SiMeCa2007} is aptly suited for visualizing this functional structure.

\begin{figure}[ht]
    \centering
        \includegraphics[width=0.85\textwidth]{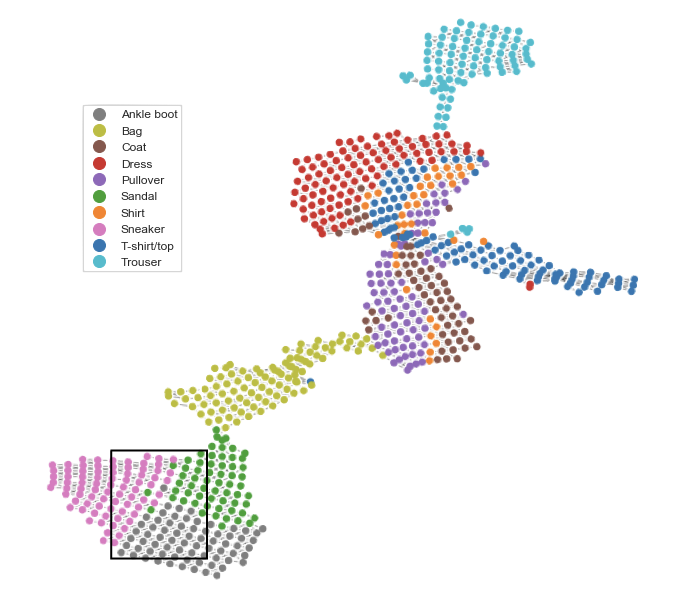}
    \caption{Mapper from logistic regression on Fashion-MNIST data.
    Window marks frame of Figure \ref{fig:multiples}.}
    \label{fig:mapper}
\end{figure}

This application is based on the work of using paths in the Mapper graph to provide explanations for supervised machine learning models \cite{saul2018explain}.
In a similar approach, Rathore, Chalapathi, Palande, and Wang \cite{RaChPaWa2021} used the Mapper construction to visually explore shapes of activations in deep learning.
They focused mostly on bifurcations in the Mapper graph, and did not directly address the question of stability of the features.
We build a Mapper from the predicted probability space of a logistic regression model.
We then extend that Mapper to be a \emph{Steinhaus Mapper Filtration} and analyze its stable paths. 

Given topological spaces $X, Y$, function $f: X \to Y$, and a cover of $Y$,
the Mapper complex is the nerve of the refined, i.e., with each cover element split into path-connected components, pullback cover of $f(Y)$.
Typically, a clustering algorithm is used to compute the pullback of each cover element.

\begin{definition}[Steinhaus Mapper filtration]
Given data $X$, a function $f: X \to Y$, and a cover $U$ of $Y$, we define the \textbf{Steinhaus Mapper} as the Steinhaus nerve (Definition \ref{def:SteinhausNerve}) of the refined pullback cover of $f(U)$:
$$ \Nrv_{St}(f^*{U}) \, . $$
\end{definition}

In order to mitigate sensitivity to parameter choice in Mapper constructions, we tested a range of resolution and gain values.

\begin{figure}[ht!]
    \centering
    \includegraphics[width=1\textwidth]{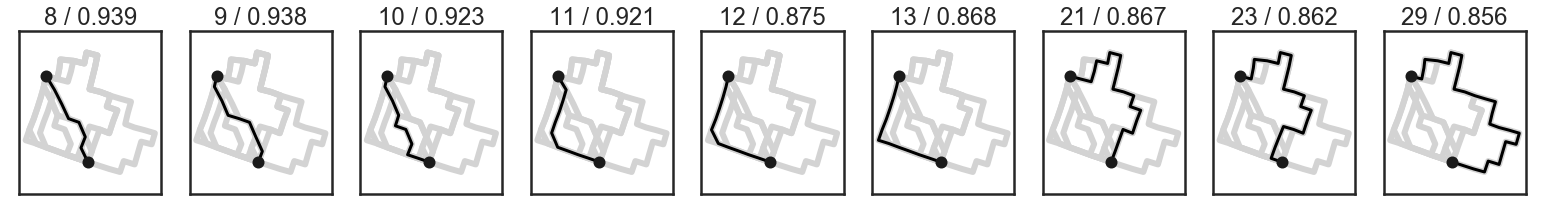}
    \caption{    \label{fig:multiples}
    Depiction of stable paths found along Pareto frontier in Figure \ref{fig:mapper_pareto}.
    }
\end{figure}

Figure \ref{fig:mapper} shows the Steinhaus Mapper filtration constructed from a logistic regression model built on the Fashion-MNIST dataset \cite{xiao2017online},
with 70,000 images of clothing items from 10 classes.
Each image is $28 \time 28$ pixels.
The Fashion-MNIST dataset is typically considered to be a more difficult application than the MNIST handwritten digits.

\begin{figure}[ht!]
  \centering
  \includegraphics[width=0.8\textwidth]{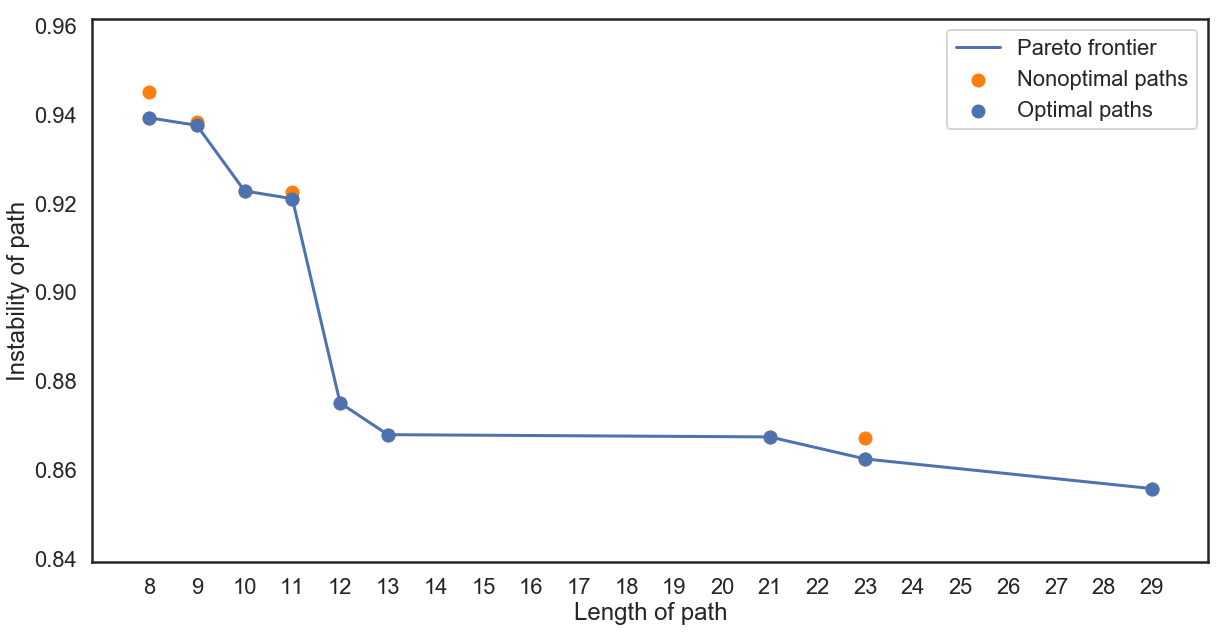}
  \caption{Pareto frontier of stable paths between predominately sneaker and ankle boot vertices. }
    \label{fig:mapper_pareto}
\end{figure}
%
We first reduce the dataset to 100 dimensions using Principal Components Analysis, and then a logistic regression classifier with $l_1$ regularization is trained on the reduced data using 5-fold cross validation on a training set of 60,000 images.
The model evaluated at 93\% accuracy on the remaining 10,000 images.
We then extract the 10-dimensional predicted probability space and use UMAP \cite{McHeMe2020} to reduce it to 2 dimensions.  
This 2-dimensional space is taken as the filter function of the Mapper complex, using a cover consisting of 40 bins along each dimension with 50\% overlap between each bin.
DBSCAN is used as the clustering algorithm in the refinement step \cite{ester1996density}.
KeplerMapper \cite{VeSaEaMa2019} is used to construct the Mapper complex.
Finally, the cover is extracted and the Steinhaus Mapper filtration is constructed.

To illustrate the power of path explanations, we start with two vertices selected from the sneaker and ankle boot regions of the resulting graph.
The three regions of shoes (sneaker, ankle boot, and sandals) are understandably confusing to the machine learning model, and we are interested in where these confusions arise.
Figure \ref{fig:multiples} shows the paths associated with the Pareto frontier (Figure \ref{fig:mapper_pareto}).

In Figure \ref{fig:mapper_pareto}, we show the Pareto frontier between the two chosen vertices.
This frontier shows a large decrease in instability value (thus increase in stability) when moving to a path length of 12.
Paths found after a large increase in stability correspond to highly stable paths (see Section \ref{sec:paths}),
i.e., the path remains the shortest while sweeping the instability value over a comparatively large range. 

Figure \ref{fig:mapper_paths} shows representatives from each vertex in the path for the shortest and the stable path with length 12.
Each row corresponds to one vertex and the columns show a representative from each class represented in the vertex.
Each image shows the multiplicity of that type of shoe in the vertex.

\begin{figure}[htp!]
    \centering
    \vspace*{-0.25in}
        \begin{tabular}{c|c}
            
            \includegraphics[width=0.25\textwidth]{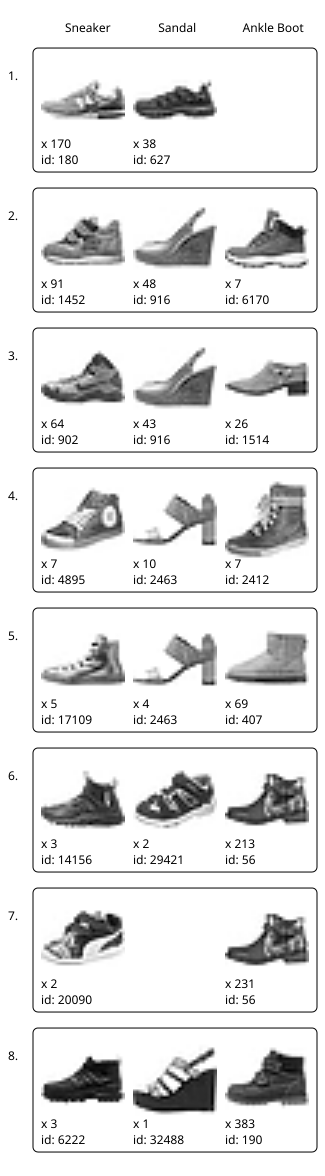}
            &  \includegraphics[width=0.25\textwidth]{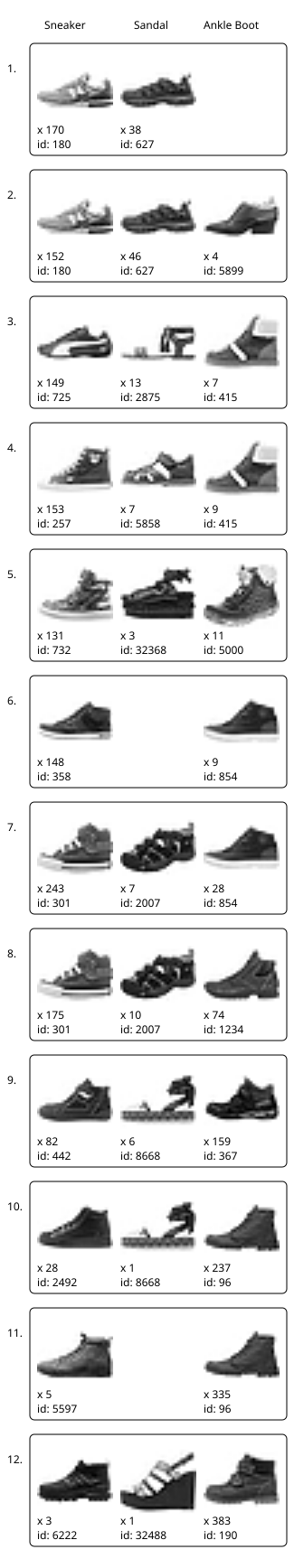}
        \end{tabular}
    \vspace*{-0.07in}
    \caption{Path visualizations for shortest path (left) and stable path with length 12 (right).
    Columns in the visualization are based on the class and each row represents a node in the Mapper graph.
    We show one representative for each node in each column.
    Columns with no shoes shown had no representative of that class in the node.}
    \label{fig:mapper_paths}
\end{figure}

In both paths, vertices start with mostly sneakers and sneaker-like sandals.
They then transition to contain larger proportions of ankle boots, with all three classes showing higher cut tops or high heels.

Along each path we see the relationships between nodes change.
In the most stable path (Right), we see a slow transition from sneaker space to ankle boot space, with some amount of sandals spread throughout.
Along the path, shoes from each of the three classes become taller.
Near the middle of the path, images from sneakers and ankle boots are indistinguishable.
Earlier in the path, we see how some white strips in the sneakers and boots might be confused with negative space in the sandals.

By exploring the two paths, we gain valuable insight into why the model is making a decision.
This can help either reinforce our trust in the model or reject the prediction.
In either outcome, these explanations can strengthen the results of the predictions by including humans in the loop.
This framework is readily applicable to much more important datasets.

%% file: sections/conclusion.tex
\label{sec:conclusion}

We introduced the Steinhaus filtration, a new kind of filtration that enables application of TDA to previously inaccessible types of data.
We then developed a theory of stable paths in the Steinhaus filtration, and provide algorithms for computing the Pareto frontier between short and stable paths.
As proof of their utility to real world applications, we show how these two ideas can be applied to the analysis of recommendation systems and of Mapper in the context of explainable machine learning. 

The results in this paper suggest many new questions. 
The application of recommendation systems leaves us curious if the Steinhaus filtration along with new results such as the one on predicting links in graphs using persistent homology \cite{bhatia2018understanding} could provide methods for answering the main question in recommendation system research: \emph{what item to recommend to the user next?}

While paths and connected components are most amenable to practical interpretations, could other structures in the Steinhaus filtration also suggest insights?
What would holes and loops in the Steinhaus filtration for recommendation systems mean?

We showed that when the cover is finite, the Steinhaus filtration is stable to small changes within the cover.
The example showing non $q$-tameness for instances with infinite covers even when it is in a totally bounded space motivates this following question: can we find a finite subcover that \emph{approximates} such infinite covers?